\documentclass{article}

\PassOptionsToPackage{numbers}{natbib}
     \usepackage[preprint]{neurips_2025}

\usepackage[utf8]{inputenc} 
\usepackage[T1]{fontenc}    
\usepackage{hyperref}       
\usepackage{url}            
\usepackage{booktabs}       
\usepackage{amsfonts}       
\usepackage{nicefrac}       
\usepackage{microtype}      
\usepackage[table]{xcolor}  
\usepackage{multirow}

\usepackage{amsmath}
\usepackage{amssymb}
\usepackage{mathtools}
\usepackage{amsthm}

\usepackage{mdframed}
\usepackage[most]{tcolorbox} 

\PassOptionsToPackage{pdftex}{graphicx}

\definecolor{lightgreen}{RGB}{220, 255, 220}

\newmdenv[
  backgroundcolor=green!5,
  linecolor=green!60!black,
  linewidth=0.5pt,
  roundcorner=2pt,
  skipabove=5pt,
  skipbelow=5pt,
  innertopmargin=6pt,
  innerbottommargin=6pt,
  innerleftmargin=6pt,
  innerrightmargin=6pt
]{greenbox}

\newmdenv[
  backgroundcolor=orange!10,
  linecolor=orange!80!black,
  linewidth=0.5pt,
  roundcorner=2pt,
  skipabove=5pt,
  skipbelow=5pt,
  innertopmargin=6pt,
  innerbottommargin=6pt,
  innerleftmargin=6pt,
  innerrightmargin=6pt
]{orangebox}

\newmdenv[
  backgroundcolor=blue!10,
  linecolor=blue!80!black,
  linewidth=0.5pt,
  roundcorner=2pt,
  skipabove=5pt,
  skipbelow=5pt,
  innertopmargin=6pt,
  innerbottommargin=6pt,
  innerleftmargin=6pt,
  innerrightmargin=6pt
]{bluebox}

\usepackage{colortbl}
\definecolor{mygrey}{RGB}{230, 230, 230}
\definecolor{mygreen}{RGB}{210, 255, 210}
\definecolor{myorange}{RGB}{255, 240, 200}

\usepackage[capitalize,noabbrev]{cleveref}

\theoremstyle{plain}
\newtheorem{theorem}{Theorem}[section]

\newtheorem{corollary}[theorem]{Corollary}
\theoremstyle{definition}

\newtheorem{remark}[theorem]{Remark}

\newcommand\R{\mathbb{R}}

\title{Rethinking Attention: Polynomial Alternatives to Softmax in Transformers}

\author{
Hemanth Saratchandran \\
\texttt{hemanth.saratchandran@adelaide.edu.au}
\And
Jianqiao Zheng \\
\texttt{jianqiao.zheng@adelaide.edu.au}
\And
Yiping Ji \\
\texttt{yiping.ji@adelaide.edu.au}
\And
Wenbo Zhang \\
\texttt{wenbo.zhang@adelaide.edu.au}
\And
Simon Lucey \\
\texttt{simon.lucey@adelaide.edu.au}
\\[1ex]
\centerline{Australian Institute for Machine Learning, University of Adelaide}
}

\begin{document}

\maketitle

\begin{abstract}
This paper questions whether the strong performance of softmax attention in transformers stems from producing a probability distribution over inputs. Instead, we argue that softmax’s effectiveness lies in its implicit regularization of the Frobenius norm of the attention matrix, which stabilizes training. Motivated by this, we explore alternative activations, specifically polynomials, that achieve a similar regularization effect. Our theoretical analysis shows that certain polynomials can serve as effective substitutes for softmax, achieving strong performance across transformer applications despite violating softmax’s typical properties of positivity, normalization, and sparsity. Extensive experiments support these findings, offering a new perspective on attention mechanisms.
\end{abstract}

\section{Introduction}

Transformer architectures \citep{vaswani2017attention} have become the foundation of state-of-the-art models across natural language processing (NLP) \citep{vaswani2017attention, devlin2018bert, zhuang2021robustly, zhen2022cosformer}, computer vision \citep{dosovitskiy2020image, carion2020end, liu2021swin, touvron2021training}, and robotics \citep{fu2024drive, maiti2023transfusion, salzmann2020trajectron++}. A core component of these models is the softmax attention block, which computes token interactions by evaluating the relative importance of inputs. This mechanism enables transformers to effectively model long-range dependencies, a capability that distinguishes them from recurrent neural networks (RNNs) and convolutional neural networks (CNNs), particularly in scaling to large datasets.

Softmax self-attention satisfies three key properties: (1) non-negativity of attention weights, (2) row-wise normalization to ensure weights sum to one (interpretable as probabilities), and (3) sparsity, promoting focus on a few relevant tokens. These properties are commonly believed to be essential for modeling effective attention \cite{bahdanau2014neural, zhen2022cosformer}, facilitating interpretability. However, this viewpoint has largely been adopted on empirical grounds, with limited theoretical justification. Despite investigations into alternative activations \citep{shen2023study, fang2022transformers, correia2019adaptively}, softmax remains dominant, primarily due to its empirical performance and interpretability.

In this paper, we challenge this prevailing assumption and ask:

\begin{greenbox}
Do attention mechanisms in transformers require non-negativity, normalization, and sparsity for effective performance?
\end{greenbox}

We present a new perspective suggesting that softmax's effectiveness stems not inherently from these properties, but from its implicit regularization of the attention matrix's Frobenius norm during training. Building on this insight, we show that simple polynomial activations, which violate non-negativity, normalization, and sparsity, can nevertheless provide comparable or superior performance to softmax by inducing a similar regularization effect. This offers a fundamentally different interpretation of attention, distinct from the conventional probabilistic view of softmax.

Our contributions are as follows:
\begin{itemize}
\item[1.] We provide a theoretical analysis showing that softmax attention implicitly regularizes the Frobenius norm of the attention matrix, challenging the assumption that non-negativity, normalization, and sparsity are the primary drivers of its success.
\item[2.] We propose polynomial activations as an alternative to softmax, demonstrating that they can induce similar regularization effects without adhering to the traditional softmax constraints, achieving competitive performance across tasks including image classification, object detection, instance segmentation, text classification, and physics-based modeling.
\end{itemize}

By revisiting foundational assumptions, our work offers deeper insights into the mechanics of attention and opens new avenues for designing alternative attention mechanisms.

\section{Related Work}
\paragraph{Attention activations.}
A variety of alternative activations for attention mechanisms have been explored in recent literature. \cite{shen2023study} showed that ReLU activations outperform softmax in long-sequence tasks, such as document translation. Of particular relevance to our work, \cite{wortsman2023replacing} demonstrated that scaling ReLU by the inverse of sequence length can surpass softmax in certain vision applications, emphasizing the importance of correct activation scaling. In this paper, we show that softmax inherently applies such a scale through its normalization, and we derive theoretical principles that motivate polynomial activations with scalings proportional to the square root of the sequence length. Other studies have proposed alternatives with varying motivations. \cite{banerjee2020exploring} used Taylor series approximations of softmax, achieving superior performance in image classification. \cite{wang2021escaping} introduced periodic activations to improve gradient flow in attention layers. \cite{koohpayegani2024sima} showed that $l^1$ normalization applied to linear attention mechanisms can yield on par performance to softmax's on three distinct vision transformers. Distinct from these works, our approach establishes a clear theoretical link between the Frobenius norm of the attention matrix and the input sequence length. Leveraging this insight, we design polynomial activations that break three canonical properties of softmax—non-negativity, row normalization, and sparsity—while still achieving competitive performance.

\paragraph{Attention mechanisms.} Numerous strategies have been proposed to improve the efficiency and scalability of transformers by reducing computational overhead and rethinking attention mechanisms. The Data-Efficient Image Transformer (DeiT) \cite{touvron2021training} leverages distillation tokens to achieve competitive performance without relying on large datasets. The Cross-Covariance Image Transformer (XCiT) \cite{ali2021xcit} introduces cross-covariance attention, enabling efficient spatial interactions with reduced complexity. The Swin Transformer \cite{liu2021swin} employs a hierarchical architecture with shifted window-based self-attention to enhance scalability for vision tasks. The Nystr\"{o}mformer \cite{xiong2021nystromformer} approximates full self-attention using the Nystr\"{o}m method, reducing its complexity from quadratic to near-linear. Similarly, the MLP-Mixer \cite{tolstikhin2021mlp} replaces self-attention entirely with multi-layer perceptrons for spatial and channel mixing. In this work, we demonstrate that our polynomial-based attention activations can be integrated into these architectures \cite{touvron2021training,ali2021xcit,liu2021swin,xiong2021nystromformer}, matching softmax performance while violating its typical properties of non-negativity, normalization, and sparsity.

\section{Preliminaries and Notation}\label{sec:prelims}

In this section we outline the definition of a transformer via the transformer block and set the notation of various mathematical quantities we will be using in future sections. For more details on transformers the reader can consult \cite{vaswani2017attention, dosovitskiy2020image}.

Transformer architectures comprise of transformer blocks, defined as follows. A transformer block is a mapping 
$\mathbf{T}: \R^{N\times D} \rightarrow \R^{N\times D}$ defined as
\begin{equation}\label{eqn:trans_main}
    \mathbf{T}(x) = \mathbf{F}(\mathbf{A}(x) + x)
\end{equation}
where $\mathbf{F}$ is a feedforward MLP with a residual connection and $\mathbf{A}$ is an attention head. 

The attention head $\mathbf{A}$ is defined as follows: It comprises of three learnable matrices, a query ($q$), key ($k$) and value ($v$) defined by: $q = QX$, $k = KX$, $v = VX$ for an input sequence $X \in \R^{N\times D}$ with 
$Q$, $K \in \R^{D\times d}$ and $V \in \R^{D\times M}$. The
attention head $\mathbf{A}(X)$ is then defined by
\begin{equation}\label{eqn:attn_eqn_general}
    \mathbf{A}(X) = \phi(\mathcal{S}(q,k))v
\end{equation}
where $\mathcal{S}$ is a similarity transformation and $\phi$ is an activation function. The most common used $\mathcal{S}$ is the dot-product: 
$\mathcal{S}(q,v) = qk^T$, known as self-attention, and will be the one we focus on in this paper. The most common activation function $\phi$ that is used by authors is softmax. This leads to the most common form of the attention head given by
\begin{equation}\label{eqn:softmax_attn}
\begin{split}
\mathbf{A}(X) &= \mathbf{softmax}\bigg(\frac{qk^T}{\sqrt{d}}\bigg)v \\
&= \mathbf{softmax}\bigg(\frac{XQK^TX^T}{\sqrt{d}}\bigg)XV.
\end{split}
\end{equation}

The function $\mathbf{softmax}$ is the matrix softmax map that applies the usual softmax function row-wise:
\begin{align}\label{eqn:softmax_matrix}
&\mathbf{softmax}\bigg{(}
\begin{bmatrix}
x_{11} & \cdots & x_{1n}\\
\vdots & \vdots & \vdots \\
x_{n1} & \cdots & x_{nn}
\end{bmatrix}
\bigg{)} 
=
\begin{bmatrix}
\frac{e^{x_{11}}}{\sum_{j=1}^ne^{x_{1j}}} & \cdots & 
\frac{e^{x_{1n}}}{\sum_{j=1}^ne^{x_{1j}}}\\
\vdots & \vdots & \vdots \\
\frac{e^{x_{n1}}}{\sum_{j=1}^ne^{x_{nj}}} & \cdots & 
\frac{e^{x_{nn}}}{\sum_{j=1}^ne^{x_{nj}}}
\end{bmatrix}
\end{align}
The factor $\frac{1}{\sqrt{d}}$, as explained in \cite{vaswani2017attention}, is a scaling to prevent the gradients of softmax from being too small. For the theoretical analysis in this paper we will only use the dot-product similarity $qk^T$ and call the $N \times N$ matrix $softmax(qk^T)$ the 
\textit{softmax self-attention matrix}. In the experiments, 
\cref{sec:exps}, we will empirically validate our theoretical framework on more general softmax attention blocks used in state of the art transformers such as DeiT \cite{touvron2021training}, Swin Transformer \cite{liu2021swin} and XciT \cite{xiong2021nystromformer}.

For general transformer architectures, multiple heads 
$\mathbf{A}_i$ for $1 \leq i \leq n$ are used. Each attention head is defined by \eqref{eqn:softmax_attn} and then all outputs of each attention head are concatenated together before going into the feedforward layer. 

We will need notation for the derivative of the matrix softmax map defined by \eqref{eqn:softmax_matrix}. Given a matrix $A \in \R^{N\times N}$ we can differentiate the matrix 
map $\mathbf{softmax}$ at $A$ and obtain the gradient linear map
$\mathbf{\nabla softmax}(A) : \R^{N\times N} \rightarrow \R^{N\times N}$ that is defined by the formula
\begin{equation}\label{eqn:softmax_derivative}
    \mathbf{\nabla softmax}(A) := \mathbf{Jsoftmax}(A)^T
\end{equation}
where $\mathbf{Jsoftmax}(A)$ is the Jacobian of $\mathbf{softmax}$ at $A$.

Given a matrix $A \in \mathbb{R}^{n \times m}$, we denote its Frobenius norm by $||A||_F$. Additionally, we use the notation $\mathbb{E}$ to represent the expectation of a random variable, where the specific random variable being considered will be clear from the context.

\section{Theoretical Analysis}\label{sec:theory}

\subsection{Implicit regulatization of Softmax}\label{subsec;sm_implicit}

This section presents a theoretical result showing that the softmax activation imposes control over the Frobenius norm of the self-attention matrix in a way that grows sub-linearly with the input sequence's token length. Additionally, we demonstrate that the gradient of the softmax with respect to the self-attention matrix also exhibits a similar degree of regularity. While previous work has analyzed the regularity of softmax self-attention through the lens of the Lipschitz constant \citep{kim2021lipschitz, castin2023understanding}, our theorem offers a novel perspective by directly linking the Frobenius norm regularity to the token length. This provides insights into how self-attention activations should scale with token length to maintain stability during training, especially with gradient descent-based algorithms.

\begin{bluebox}
\begin{theorem}\label{thm:softmax_regularity}
Let $\mathbf{softmax} : \R^{N\times N} \rightarrow 
\R^{N\times N}$ be the matrix softmax map defined by  
\eqref{eqn:softmax_matrix} and let 
$\mathbf{\nabla softmax}(A) : \R^{N\times N} \rightarrow 
\R^{N\times N}$ denote the gradient of 
$\mathbf{softmax}$ at $A \in \R^{N\times N}$. We then have the following bounds on the Frobenius norms
\begin{align}
 ||\mathbf{softmax}(A)||_F &\leq \sqrt{N} \\
   ||\mathbf{\nabla softmax}(A)||_F &\leq 
   2\sqrt{N}.
\end{align}
\end{theorem}
\end{bluebox}

The key implication of theorem \ref{thm:softmax_regularity} is that during the training of a transformer with softmax self-attention, the Frobenius norm of each softmax self-attention matrix remains bounded by a value that grows as $\mathcal{O}(\sqrt{N})$. This ensures that backpropagation through the weights of the self-attention matrix does not lead to excessively large gradients. The proof hinges on the fact that the row normalization inherent in softmax effectively controls the Frobenius norm. For a detailed proof see \cref{app;proofs_sm}.

\subsection{Polynomial activations for self-attention}\label{subsec:poly_acts}

In \cref{subsec;sm_implicit}, we demonstrated that softmax implicitly regularizes the Frobenius norm of the self-attention matrix. Building on this, we now show that by scaling specific polynomial activations, a similar regularization effect on the Frobenius norm can be achieved in expectation, closely replicating the impact of softmax.

\begin{bluebox}
\begin{theorem}\label{thm:expectation}
    Let $X \in \R^{N \times D}$ and $Q$, $K \in \R^{D\times d}$ be i.i.d random variables distributed according to
    $X \sim \mathcal{N}(0, \sigma_x)$ and 
    $Q$, $K \sim \mathcal{N}(0, \sigma_t)$. We have the following expectations of the Frobenius norms of powers of the $N \times N$ matrix $(XQK^TX^T)^p$ for $p \geq 1$
\begin{equation}
        \mathbb{E}\bigg{|}\bigg{|}\bigg{(}
    \frac{XQK^TX^T}{\sqrt{d}} \bigg{)}^p
    \bigg{|}\bigg{|}_F \leq 
    \mathcal{O}(N)
\end{equation}
\end{theorem}
\end{bluebox}

By scaling such an activation by $\frac{1}{\sqrt{N}}$ we can obtain a $\mathcal{O}(\sqrt{N})$ bound.

\begin{orangebox}
\begin{corollary}\label{cor:expectation_tight}
    Assume the same conditions as in theorem \ref{thm:expectation}.
    Then
    \begin{equation}
     \mathbb{E}\bigg{|}\bigg{|}\frac{1}{\sqrt{N}}\bigg{(}
    \frac{XQK^TX^T}{\sqrt{d}} \bigg{)}^p
    \bigg{|}\bigg{|}_F \leq \mathcal{O}(\sqrt{N}).
    \end{equation}
\end{corollary}
\end{orangebox}
\cref{cor:expectation_tight} establishes that activations of the form $\phi(x) := \frac{1}{\sqrt{N}}x^p$ provide a level of regularization, in expectation, similar to that of softmax when applied to the self-attention matrix. The proof of theorem \ref{thm:expectation} can be found in appendix \ref{app;poly_proofs}.
The next property we want to prove is one similar to the gradient bound obtained in 
\cref{thm:softmax_regularity}. Since the self-attention matrix has parameters given by the queries $Q$ and keys $K$ \citep{vaswani2017attention}, this implies that during the training of a transformer the $Q$ and $K$ matrices are the only aspects of the self-attention matrix that get updated. Therefore, we compute a derivative bound with respect to the $Q$ and $K$ derivatives.

\begin{bluebox}
\begin{theorem}\label{thm:grad_expectation}
    Let $X \in \R^{N \times D}$ and $Q$, $K \in \R^{D\times d}$ be i.i.d random variables distributed according to
    $X \sim \mathcal{N}(0, \sigma_x)$ and 
    $Q$, $K \sim \mathcal{N}(0, \sigma_t)$. Then the expectation of the of the derivative of the matrix
    $\frac{(XQK^TX^T)^p}{\sqrt{d}}$ w.r.t the $Q$ parameter matrix for $p \geq 1$ is given by
\begin{equation}
    \mathbb{E}\bigg{|}\bigg{|}
    \frac{\partial}{\partial Q}\bigg{(}
    \frac{(XQK^TX^T)^p }{\sqrt{d}}
    \bigg{)}
    \bigg{|}\bigg{|}
    \leq \mathcal{O}(N)
\end{equation}
\end{theorem}
\end{bluebox}

\begin{figure}[ht!]
    \centering
    \includegraphics[width=0.5\linewidth]
    {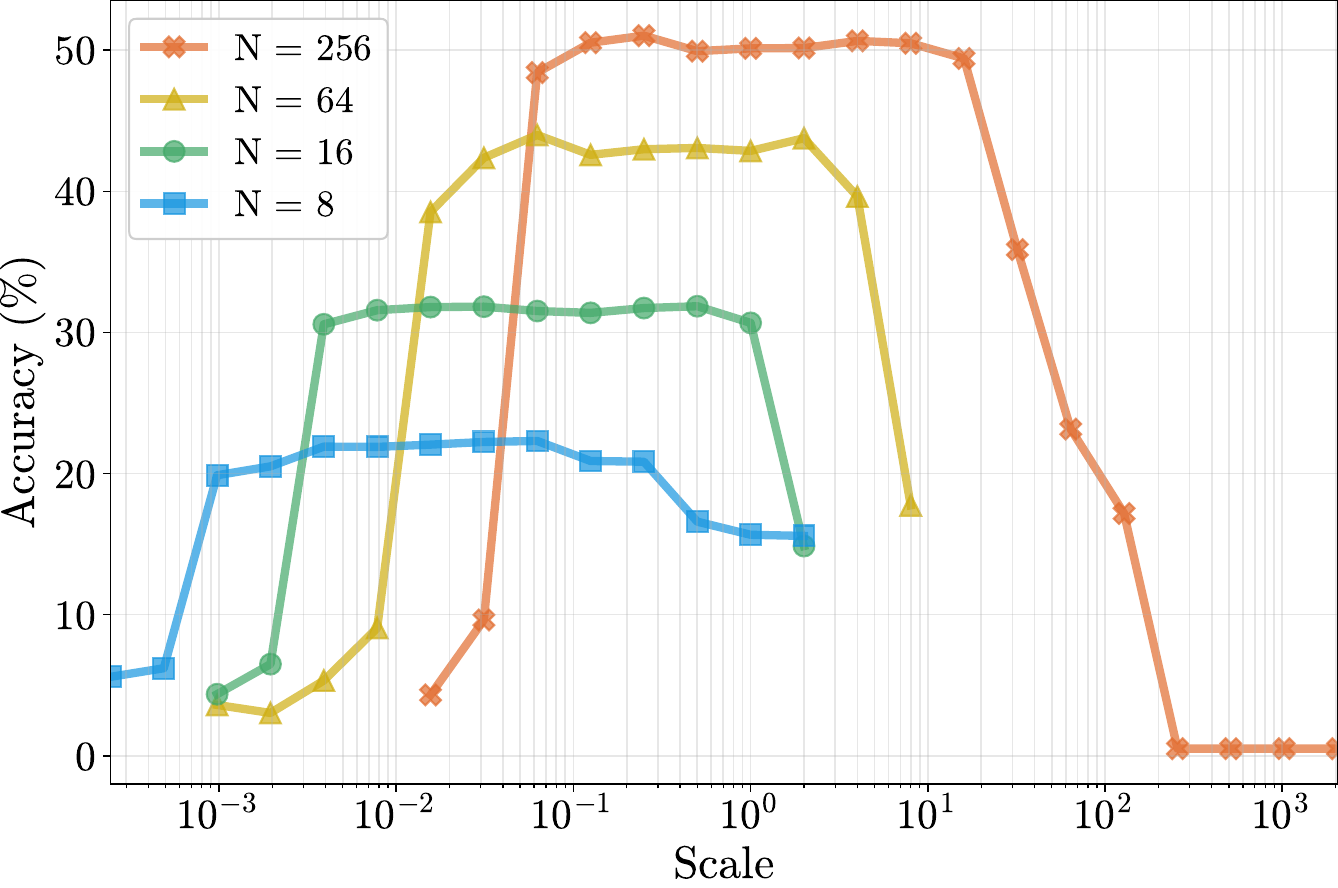}
    \vspace{-1em}
    \caption{Training ViT-Tiny with the activation $\phi(x) = x^3$ with different sequence lengths and different scales. As the sequence length gets larger, the $k$ scale (x-axis) needed to obtain good accuracy when using $\frac{1}{k}x^3$ as an activation increases validating the theory from section \ref{subsec:poly_acts}.}
    \label{fig:cube_imgsize}
\end{figure}

\begin{orangebox}
\begin{corollary}\label{cor:grad_expectation_tight}
Assume the same condition as in theorem \ref{thm:grad_expectation}. Then
\begin{equation}
\mathbb{E}\bigg{|}\bigg{|}\frac{1}{\sqrt{N}}
    \frac{\partial}{\partial Q}\bigg{(}
    \frac{(XQK^TX^T)^p }{\sqrt{d}}
    \bigg{)}
    \bigg{|}\bigg{|}
    \leq \mathcal{O}(\sqrt{N}).
\end{equation}
\end{corollary}
\end{orangebox}

An analogous estimate holds for derivatives with respect to the $K$ matrix. The proof of theorem \ref{thm:grad_expectation} can be found in appendix \ref{app;poly_proofs}.

\begin{remark}
\cref{cor:expectation_tight} and \cref{cor:grad_expectation_tight} suggest that polynomial activations of the form $\phi(x) = \frac{1}{\sqrt{N}}x^p$, with $p > 0$, can achieve performance comparable to softmax when applied to self-attention matrices. We point out that both corollaries rested on the assumption that $X$, $Q$ and $K$ are i.i.d random variables. In general, this is only true at initialization when training a transformer \cite{albert2025randlora}. Although this is a limitation in the theory the experiments in \cref{sec:exps} show that this insight can be used effectively to develop new attention blocks that perform comparable to softmax yet violate the three conditions of positivity, normalized rows summing to $1$ and sparsity showing that attention blocks do not need to be modeled as a probability distribution.    
\end{remark}

\section{Testing the theory}\label{subsec:testing_theory}

In this section we test the theory developed in \cref{subsec:poly_acts} on small vision transformers.
We will consider the activation 
$\phi(x) =  \frac{1}{k}x^3$, where $k > 0$ is a fixed scale, as this activation clearly violates the three key conditions of softmax based attention; positivity, normalization and sparsity. We found that polynomials $\frac{1}{k}x^p$ for $p > 3$ did not perform well during training as they witnessed a gradient vanishing problem due to the fact that the function $x^p$ for $p$ large have very small values around $0$.

The first experiment we conducted was to test how the Top-$1\%$ accuracy changes for a ViT-Tiny vision transformer \cite{steiner2021train}, trained on Tiny-Imagenet dataset \cite{tinyimagenet2015},
as we change the sequence length $N$ of the input and the scale predicted in corollaries \ref{cor:expectation_tight} and \ref{cor:grad_expectation_tight} when using the activation 
$\frac{1}{k}x^3$. The standard ViT-Tiny model comprises $12$ transformer layers, each equipped with $3$ attention heads, with each head having dimension $64$.
We considered four different input sequence lengths $N$ of sizes $256$, $64$, $16$ and $8$. For each such sequence length, we ran a ViT-Tiny architecture with the activation $\frac{1}{k}x^3$
where $k$ ranged from roughly $10^{-3}$ to $10^3$. 
According to the theory developed in section \ref{sec:theory}, the Frobenius norm of $\frac{1}{\sqrt{N}}x^3$ scales according to $\mathcal{O}(\sqrt{N})$. Thus the best accuracy should occur when $k = \mathcal{O}(\sqrt{N})$ and should degrade for other values due to training instability. \cref{fig:cube_imgsize} shows the results of the experiment, we note that the $x-axis$ plots the values of $k$. We see that as the sequence length increases the factor of $k$ needs to increase so that the activation $\frac{1}{k}x^3$ performs well on the ViT-Tiny architecture as predicted by the theory developed in corollaries \ref{cor:expectation_tight} and \ref{cor:grad_expectation_tight}.

In a second experiment we decided to compare the performance of the original ViT-Tiny architecture, that uses an input sequence length of $256$, with a softmax activation and a polynomial activation on the Tiny-ImageNet dataset.
In 
\cref{cor:expectation_tight} and \cref{cor:grad_expectation_tight} it was pointed out that the scaling of the polynomial is important to keep the Frobenius norm of the polynomial attention matrix from becoming too large. The results of those corollaries suggested that a scale to use is 
$\frac{1}{\sqrt{N}}$, $N$ being the sequence length, which in this case is 
$\frac{1}{\sqrt{256}} = \frac{1}{16}$. We therefore decided to
to compare the three activations softmax, $x^3$ and $\frac{1}{16}x^3$. Each was trained for 200 epochs using the AdamW optimizer.
To begin with we computed the Frobenius norm of the attention matrix throughout training averaged over all the heads in layers 2, 7 and 12 of the ViT-Tiny architecture when trained on the Tiny-ImageNet dataset. \cref{fig:attention_norm_tiny} plots the results. We see from the figure that the Frobenius norm of of the $x^3$ archicture is much larger than softmax but scaling it by 
$\frac{1}{16}$ brings it down to softmax levels.
Similarly, \cref{fig:jac_norm_tiny} shows the Jacobian's Frobenius norm, where scaling also brings the norms closer to softmax, ensuring more stable gradients. Further plots for other layers are in \cref{app:frob_norm_complete}. 
\cref{tab:tinyimagenet_results} presents the final Top-1\% accuracy achieved by each activation function. Notably, $\frac{x^{3}}{16}$ delivers the best performance. In contrast, the unscaled $x^3$ activation yields significantly lower Top-1\% accuracy, underscoring the importance of incorporating an appropriate scaling factor.

\begin{figure*}[ht!]
    \centering
    \includegraphics[width=1.0\linewidth]
    {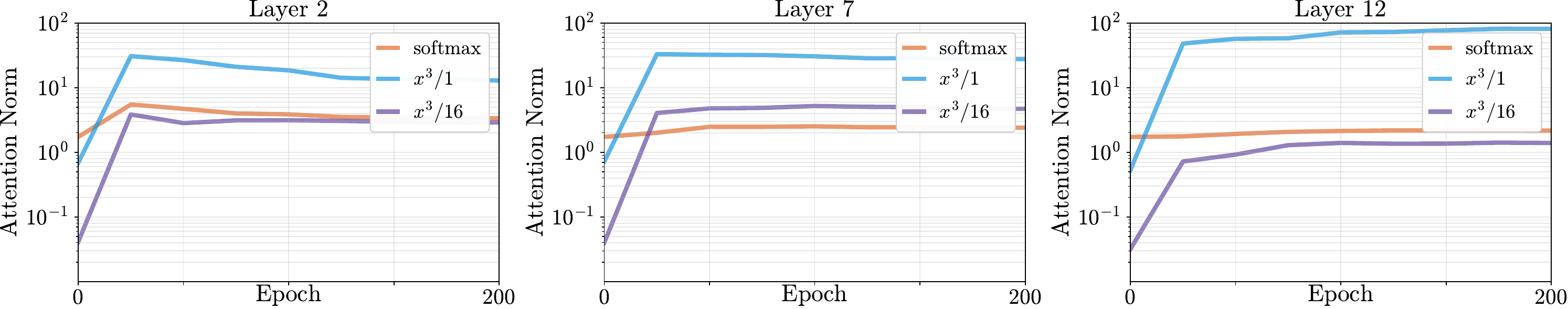}
    \caption{Frobenius norm of the self-attention matrix with three different activations in layer 2, 7 and 12 of the ViT-Tiny architecture during training.}
    \label{fig:attention_norm_tiny}
\end{figure*}

\begin{figure*}[ht!]
    \centering
    \includegraphics[width=1.0\linewidth]
    {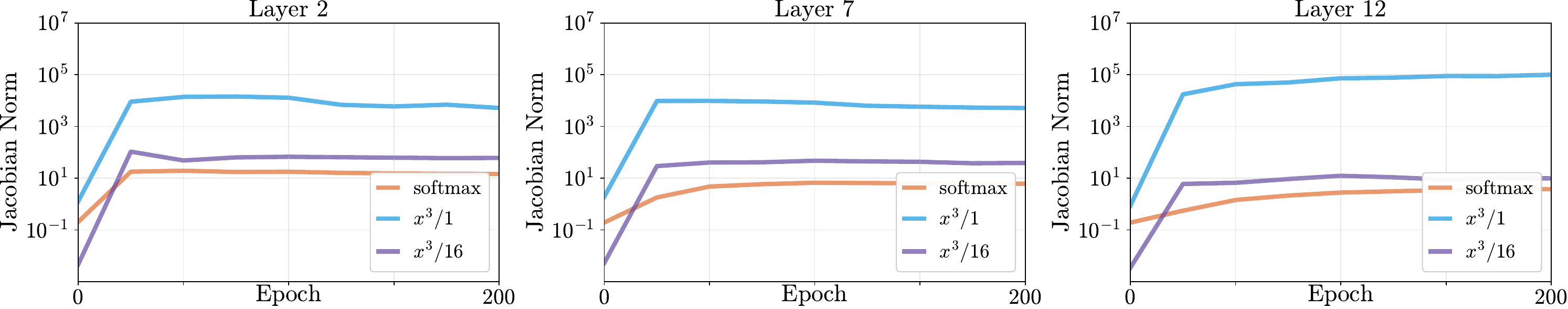}
    \vspace{-1em}
    \caption{Frobenius norm of the jacobian of the self-attention matrix with three different activations in layer 2, 7 and 12 of the ViT-Tiny architecture during training.}
    \label{fig:jac_norm_tiny}
\end{figure*}

\begin{table}[!ht]
\caption{Comparison of Top-1\% accuracy on Tiny-ImageNet between softmax and polynomial activations.}
    \centering
    \begin{tabular}{c|c c c  }

         & softmax & $\frac{x^{3}}{16}$ & $x^3$ \\
        \midrule
        Top-1\% accuracy & 50.26 & \textbf{50.5} & 45.3 \\

        \midrule
    \end{tabular}
    \vspace{0.1cm}
    \label{tab:tinyimagenet_results}
\end{table}

\section{Experiments}\label{sec:exps}

In this section, we evaluate a simple polynomial activation as an alternative to the standard softmax across a range of transformer applications commonly studied in the literature. The goal is to empirically challenge the conventional softmax properties, positivity, row-normalization, and sparsity—and examine whether these conditions are truly necessary for good transformer performance.

Building on the theoretical foundations from \cref{sec:theory}, we focus on the cubic polynomial $x^3$ as a test case. Notably, $x^3$ introduces both positive and negative values, does not normalize rows to sum to 1, and generally produces dense attention matrices—violating all three traditional softmax conditions.

We consider two scaling strategies for $x^3$:
\begin{itemize}
\item[1.] \textbf{Fixed scale:} Following \cref{sec:theory}, we scale $x^3$ by the inverse square root of the sequence length (which is fixed throughout training), as theory suggests this maintains optimization stability.
\item[2.] \textbf{Dynamic scale:} Recognizing that the assumption of i.i.d. normal-distributed $Q$, $K$, $V$, used in \cref{sec:theory},  holds primarily at initialization, we also explore a learnable scale. This dynamic scale is initialized as above but optimized during training.
\end{itemize}

Our experiments are not designed to achieve state-of-the-art results. Instead, they aim to question the softmax paradigm and show that alternative activations, even those violating softmax’s traditional properties, can still lead to effective transformer models. For each experiment we ran five trials and report the mean result.

\subsection{Image Classification}\label{subsec:image_classi}

We conducted an image classification task using various vision transformer architectures from the literature on the ImageNet-1k dataset. For this task, the standard sequence length employed by the vision transformers on the ImageNet-1k dataset is $196$.


We trained all models on the ImageNet-1k dataset from scratch and report Top-1 accuracy on the validation set. We examined our approach along with the following four transformer architectures to show its generalization, ViT-B \cite{dosovitskiy2020image}, DeiT \cite{touvron2021training}, Swin Transformer \cite{liu2021swin}, XCiT \cite{ali2021xcit}. Each transformer was trained following the approach in each paper.
The results are shown in \cref{tab:vits}. As can be seen from the table the $x^3$ activation with a dynamic scale performed the best and the one with a fixed scale performed comparable to softmax. On the other hand $x^3$ (with no scale) underperformed on all ViTs showing how important the theory on scaling as developed in \cref{sec:theory} is.
For a discussion on other polynomials and linear attention see \cref{app:other_polys}.




\begin{table}[!ht]
    \caption{Comparisons of pre-training models with different activation functions on ImageNet-1k. We report top-1 classification accuracy (\%).}
    \centering
    \scalebox{1.0}{
    \begin{tabular}{ccccc}
        \rowcolor{mygrey}
        \multirow{2}{*}{} & \multicolumn{4}{c}{Models} \\
        \rowcolor{mygrey}
        & ViT-B & DeiT-B & Swin-B & XciT-M \\
        \midrule
        softmax & 80.3 & 81.5 & 83.5 & 82.7 \\
        \midrule
        $x^3$ + fixed scale & 80.2 & 81.4 & \cellcolor{myorange}83.6 & \cellcolor{myorange}82.8 \\
        \midrule
        $x^3$ + dynamic scale & \cellcolor{mygreen}80.3 & \cellcolor{mygreen}{81.6} & \cellcolor{mygreen}{83.6} & \cellcolor{mygreen}{82.9} \\
        \midrule
        $x^3$ & 78.1 & 78.5 & 79.9 & 79.5 \\
        \midrule
    \end{tabular}
    }
    \label{tab:vits}
\end{table}

\paragraph{Visualizing attention heads:} Softmax attention traditionally satisfies three key properties: positivity, row normalization, and sparsity. In contrast, the polynomial activation $x^3$ (with or without positive scaling) takes positive values for $x > 0$ and negative values for $x < 0$, thus violating these softmax constraints. To better understand how this affects attention patterns, we analyzed the self-attention matrices of ViT-B models trained with the $x^3$ + dynamic scale activation and compared them to those using softmax. We visualized heatmaps of the attention matrices after convergence, focusing on two representative layers and heads, and averaged over a fixed batch of 128 samples. \cref{fig:heat_2} shows results for layer 2, head 8, where the $x^3$ + dynamic scale activation produces attention scores with both positive and negative values, unlike softmax. Similarly, \cref{fig:heat_12} illustrates distinct patterns for layer 12, head 6, highlighting how the two activations differ in learned attention distributions.

\paragraph{Interpretability.} Even with these differences, the $x^3$ + dynamic scale activation still learns meaningful attention patterns. For instance, in \cref{fig:heat_2} (left), we observe 14 bands about the diagonal, indicating that the head has learned to attend to patches in the same image row as the query patch, sufficient for effective image classification. This suggests that the sign of attention values (positive or negative) is not inherently critical for attention allocation. Further, \cref{fig:heat_12} (left) reveals vertical lines, showing attention that depends solely on key positions, independent of the query position. Smaller key indices receive higher attention weights, focusing model capacity where it matters most for classification tasks. These findings demonstrate that, despite deviating from softmax properties, the $x^3$ + dynamic scale activation enables the model to discover effective attention patterns. We noticed similar attention patterns for the
$x^3$ + fixed scale activation.


\begin{figure}[ht!]
    \centering
    \includegraphics[width=1.\linewidth]
    {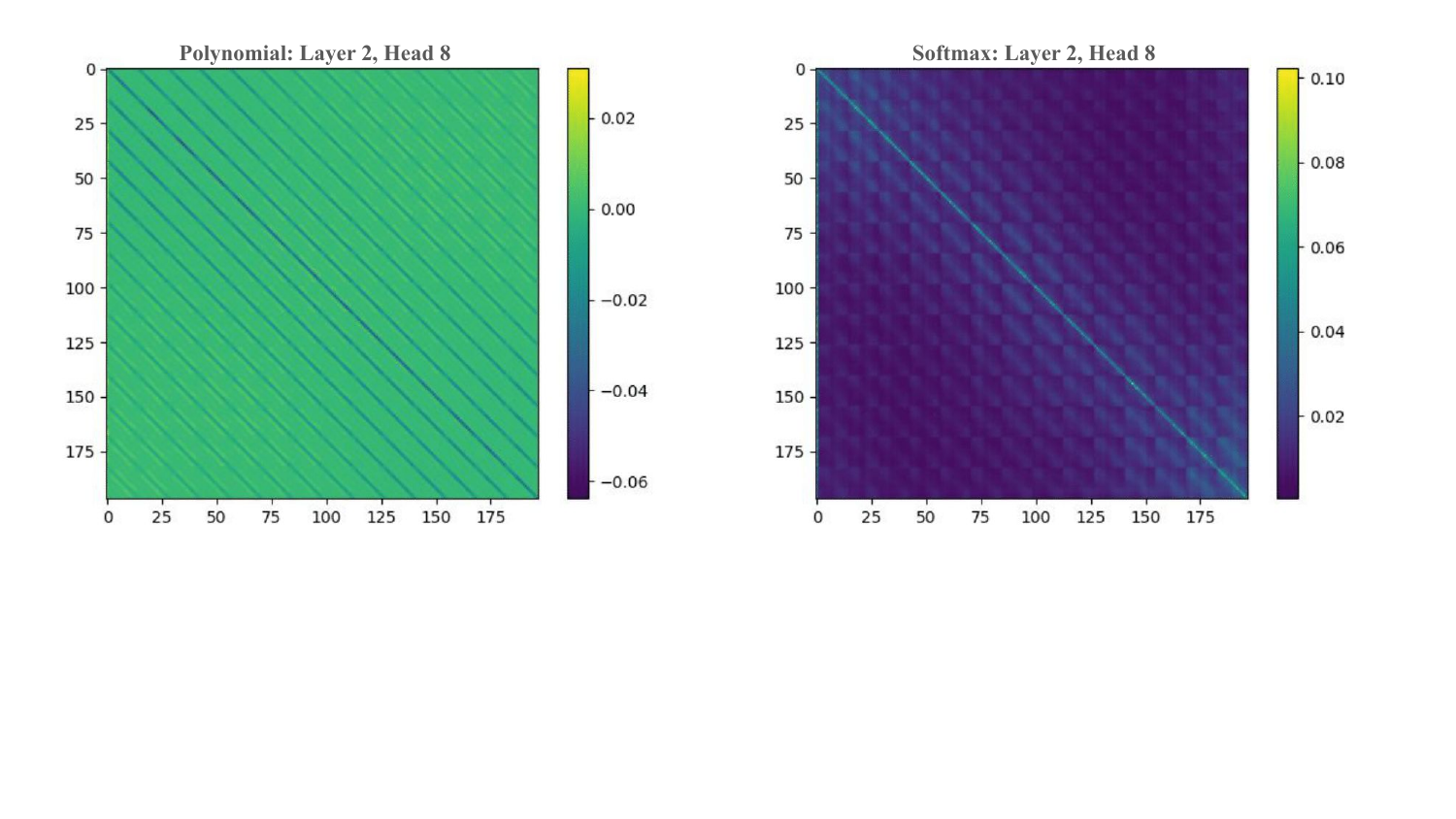}
    \vspace{-9em}
    \caption{Heat maps of the self-attention matrix in layer 2, head 8, of a ViT base architecture, comparing $x^3$ + dynamic (left) and softmax (right) activations after training. The stark difference in self-attention patterns between the two activations is evident, showing distinct distributions across input tokens.}
    \label{fig:heat_2}
\end{figure}

\begin{figure}[ht!]
    \centering
    \includegraphics[width=1.0\linewidth]
    {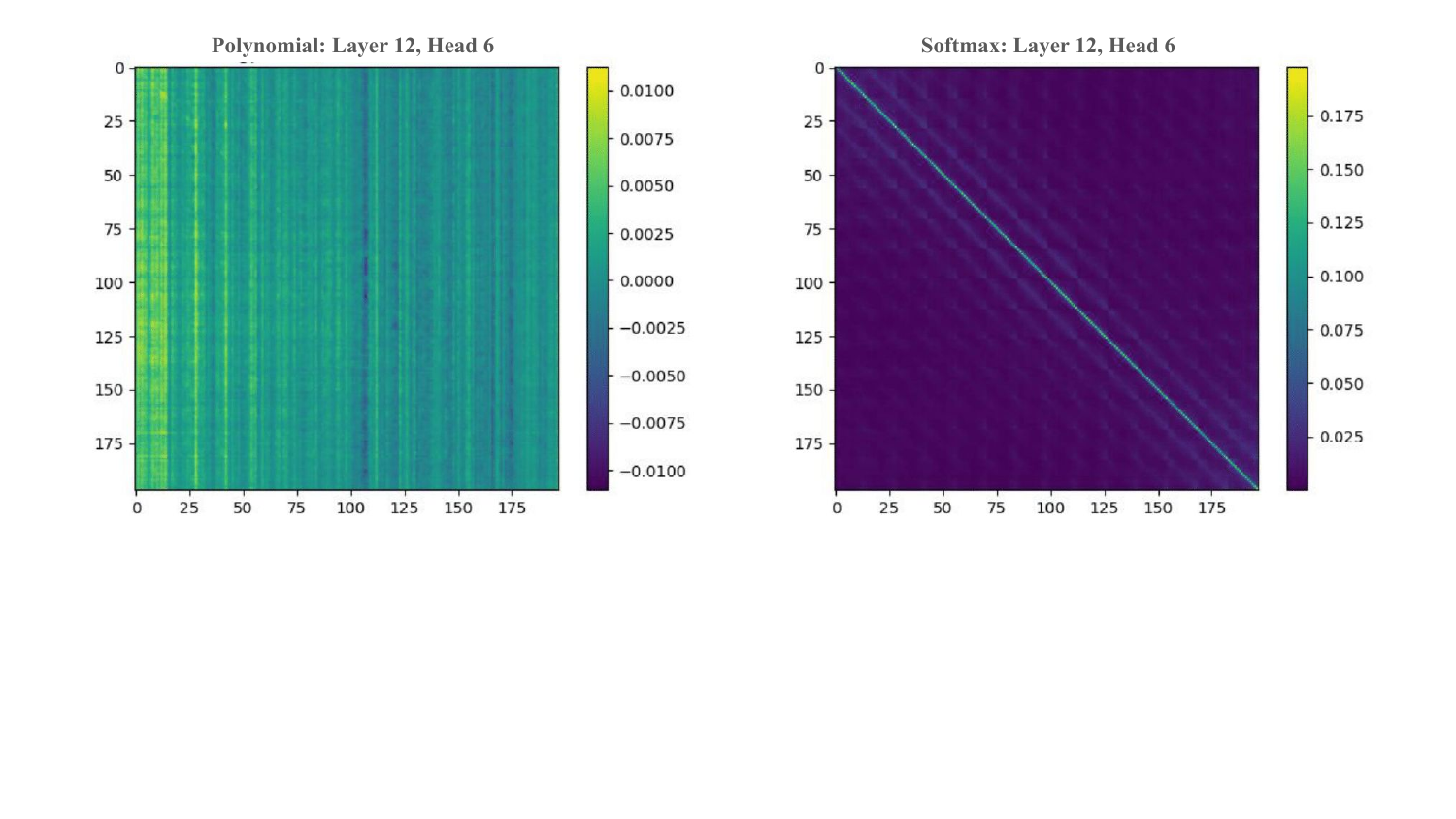}
    \vspace{-9em}
    \caption{Heat maps of the self-attention matrix in layer 12, head 6, of a ViT base architecture, comparing $x^3$ + dynamic scale (left) and softmax (right) activations after training. The contrast in self-attention patterns between the two activations is clearly visible.}
    \label{fig:heat_12}
\end{figure}

\subsection{Object Detection and Instance Segmentation}\label{subsec:object_detec}

In this section, in order to examine the transfer learning ability of our models, we demonstrate our approach on object detection and segmentation tasks by fine-tuning a ImageNet-pretrained XCiT model on them. Our experiments are conducted on COCO 2017 \cite{lin2014microsoft}, which has 118K training images and 5K validation images with 80 categories. We integrate the XCiT architecture as the backbone in the Mask R-CNN \citep{he2017mask} detector with a Feature Pyramid Network (FPN). Due to XCiT's inherently columnar design, we adapt it for FPN compatibility by extracting features from various layers for XCiT-S12. These features have a consistent stride of 16. The feature resolutions are then adjusted to strides of [4, 8, 16, 32]. This downsampling is accomplished through max pooling, while upsampling is achieved using a single transposed convolution layer. 
We conducted experiments on XCiT-S12 models using 16×16 patches with the activations $x^3$ + fixed scale, $x^3$ + dynamic scale, with scale initialized to 
$\frac{1}{14}$, and softmax. We found when we trained with $x^3$ the model would not converge thereby suggesting that a scaling is extremely important as predicted by the theory in \cref{subsec:poly_acts}. The training hyperparameters are given in 
\cref{app:exps_hyps}.
\cref{tab:transferlearning} shows the results of the experiment. From the table we see that the activation $x^3$ + dynamic scale obtains better performance to softmax with $x^3$ + fixed scale obtaining comparable performance.




\begin{table*}[!ht]
    \caption{COCO object detection and instance segmentation performance on the mini-val set. All backbones are pretrained on ImageNet-1k, and use Mask R-CNN model. $AP^b$: Average Precision for bounding box predictions, $AP^b_{50/75}$: Average Precision at an IoU threshold of 0.50/0.75 for bounding box predictions, $AP^m$: Average Precision for mask predictions, $AP^m_{50/75}$: Average Precision at an IoU threshold of 0.50/0.75 for mask predictions.}
    \centering
    \setlength{\tabcolsep}{12pt}
    \begin{tabular}{c c c c c c c}
        \rowcolor{mygrey}
        {} & $AP^b$ & $AP^b_{50}$ & $AP^b_{75}$ & $AP^m$ & $AP^m_{50}$ & $AP^m_{75}$ \\
        \midrule
        softmax & 44.9 & 66.1 & 48.9 & 40.1 & 63.1 & 42.8 \\
        \midrule
        $x^3$ + fixed scale & 44.8 & \cellcolor{myorange}66.3 & \cellcolor{myorange}49.1 & \cellcolor{myorange}40.2 & \cellcolor{myorange}63.1 & \cellcolor{myorange}43.0 \\
        \midrule
        $x^3$ + dynamic scale & \cellcolor{mygreen}45.1 & \cellcolor{mygreen}66.5 & \cellcolor{mygreen}49.4 & \cellcolor{mygreen}40.4 & \cellcolor{mygreen}63.2 & \cellcolor{mygreen}43.1 \\
        \midrule
    \end{tabular}
    \label{tab:transferlearning}
\end{table*}

\subsection{Nystr\"omformer on LRA benchmark}\label{subsec:nlp_exps}

Nystr\"omformer is a transformer designed to handle long-range dependencies more efficiently by approximating the self-attention mechanism with the Nystr\"{o}m method \cite{xiong2021nystromformer}. 
To assess the effectiveness of our approach, we trained models on five benchmarks from the Long Range Arena (LRA) suite \cite{tay2020long}: ListOps, Text Classification, Retrieval, Image Classification, and Pathfinder.

On each dataset we compared a Nystr\"omformer with softmax, 
$x^3$ + dynamic scale, and $x^3$. For the $x^3$ + dynamic scale activation on each dataset we used a different scale as initialization based on the sequence length used in the original Nystr\"omformer model. For ListOps the original model used a sequence length of $512$ and so the scale we used as initialization was $1/\sqrt{512}$, on Text Classification the input sequence used is $1024$ and so the scale we used at initialization was $1/\sqrt{1024}$, on Retrieval the input sequence used is $2048$ and the initialized scale we used was 
$1/\sqrt{2048}$, on Image Classification the input sequence is 
$4096$ and the initialized scale we used was $1/\sqrt{4096}$, and on Pathfinder the input sequence used is 
$8192$ and the initialized scale we used was $1/\sqrt{8192}$. We followed the training regime of the original paper \cite{xiong2021nystromformer}. The results are shown in \cref{tab:nystrom}. As can be seen from the table the $x^3$ + dynamic scale activation results in better performance compared to softmax with $x^3$ + fixed scale performing comparably. However $x^3$ (with no scale) underperforms on all tasks.

\begin{table*}[!ht]
    \caption{Comparisons of Nystr\"omformer models with different activation functions on the LRA benchmark. We report the accuracy (\%).}
    \centering
    \setlength{\tabcolsep}{12pt}
    \begin{tabular}{cc c c c c }
        \rowcolor{mygrey}
        {} & ListOps & Text & Retrieval & Image & Pathfinder \\
        \midrule
        softmax & 37.1 & 63.8 & 79.8 & 39.9 & 72.9 \\
        \midrule
        $x^3$ + fixed scale & \cellcolor{myorange}37.2 & 63.7 & \cellcolor{myorange}80.0 & \cellcolor{myorange}39.9 & \cellcolor{myorange}72.9 \\
        \midrule
        $x^3$ + dynamic scale & \cellcolor{mygreen}37.5 & \cellcolor{mygreen}63.9 & \cellcolor{mygreen}81.1 & \cellcolor{mygreen}40.1 & \cellcolor{mygreen}73.1 \\
        \midrule
        $x^3$ & 32.3 & 62.0 & 78.5 & 38.1 & 67.9 \\
        \midrule
    \end{tabular}
    \label{tab:nystrom}
\end{table*}

\section{Physics Informed Transformers}

Recently, dynamical data has become essential for modeling time-dependent systems. \cite{zhao2023pinnsformer} introduced transformer architectures for solving partial differential equations (PDEs), demonstrating superior performance across various PDE tasks. These models effectively capture interactions between spatial and temporal components inherent in dynamical systems. In \cref{sec:pinns}, we extend our polynomial attention framework to Physics-Informed Transformers, showing that polynomial activations can outperform softmax in this setting.

\section{Limitations}\label{sec:limitations}
While our work introduces polynomial activations that challenge the conventional understanding of softmax, there are several limitations to address. First, our theoretical framework is specifically developed for dot-product self-attention and may not directly generalize to other forms of attention mechanisms, such as additive attention or kernelized approximations. Extending our analysis to these variants could reveal additional nuances. Second, while our empirical evaluations span multiple architectures and tasks, they are limited to models with up to 100 million parameters due to resource constraints. The scalability of polynomial activations in large-scale transformers, especially those with billions of parameters, should be considered in the future.


\section{Conclusion}\label{sec:conclusion}

This work questioned whether transformer activations for attention must produce sparse probability distributions. We introduced a theoretical framework analyzing the Frobenius norm of the self-attention matrix, which suggests key scaling properties for activations in attention mechanisms. We proved that specific polynomial activations, which behave very differently from softmax, satisfy these properties. Through extensive experiments across a variety of transformer applications, we demonstrated that these alternative activations not only compete with but can outperform softmax, offering a new perspective on attention mechanisms in transformers.

\section{Acknowledgements}
Simon Lucey acknowledges support from the Australian Research Council (ARC) through the Discovery Project DP220103803.


\newpage
{
    \bibliographystyle{plain}
    \bibliography{main}
}


\appendix

\section{Appendix}
\subsection{Theoretical analysis}\label{app:theory_analysis}

\subsubsection{Proofs for theorems in section \ref{subsec;sm_implicit}}\label{app;proofs_sm}

In this section we give the proof of theorem \ref{thm:softmax_regularity}.

\begin{proof}[Proof of theorem \ref{thm:softmax_regularity}]

We will start by proving the first inequality in theorem \ref{thm:softmax_regularity}. Given a matrix $A = (a_{ij}) \in \R^{N\times N}$ we have that
\begin{equation}\label{app_eqn:softmax_acting_A}
    \mathbf{softmax}\bigg{(}
\begin{bmatrix}
a_{11} & \cdots & a_{1n}\\
\vdots & \vdots & \vdots \\
a_{n1} & \cdots & a_{nn}
\end{bmatrix}
\bigg{)}
=
\begin{bmatrix}
\frac{e^{a_{11}}}{\sum_{j=1}^ne^{a_{1j}}} & \cdots & 
\frac{e^{a_{1n}}}{\sum_{j=1}^ne^{a_{1j}}}\\
\vdots & \vdots & \vdots \\
\frac{e^{a_{n1}}}{\sum_{j=1}^ne^{a_{nj}}} & \cdots & 
\frac{e^{a_{nn}}}{\sum_{j=1}^ne^{a_{nj}}}.
\end{bmatrix}
\end{equation}
By definition of the Frobenius norm we then see that 
\begin{align}
    ||\mathbf{softmax}(A)||_F^2 &= \bigg{(}\frac{1}{\sum_{j=1}^ne^{a_{1j}}}\bigg{)}^2(e^{2a_{11}} + \cdots e^{2a_{11}}) + \cdots \\
    &\hspace{1cm} +
    \bigg{(}\frac{1}{\sum_{j=1}^ne^{a_{Nj}}}\bigg{)}^2(e^{2a_{N1}} + \cdots e^{2a_{NN}}) \\
    &\leq 
    \bigg{[}\bigg{(}\frac{1}{\sum_{j=1}^ne^{a_{1j}}}\bigg{)}(e^{a_{11}} + \cdots e^{a_{11}})\bigg{]}^2 + \cdots  \\
    &\hspace{1cm} +
    \bigg{[}\bigg{(}\frac{1}{\sum_{j=1}^ne^{a_{Nj}}}\bigg{)}(e^{a_{N1}} + \cdots e^{a_{NN}})\bigg{]}^2 \\
    &= 1 + \cdots + 1 \\
    &= N
    \end{align}
where the second inequality uses the fact that for non-negative numbers $a$ and $b$ we always have that $a^2 + b^2 \leq (a+b)^2$.

It then immediately follows that $||\mathbf{softmax}(A)||_F \leq \sqrt{N}$ and this proves the first inequality in the statement of theorem \ref{thm:softmax_regularity}.

We move on to prove the second inequality in the statement of theorem \ref{thm:softmax_regularity}. For this, let us write each entry of the matrix on the right of \eqref{app_eqn:softmax_acting_A} as follows:
\begin{equation}
    F_{kl} = \frac{e^{a_{kl}}}{\sum_{j=1}^Ne^{a_{kj}}}.
\end{equation}
By applying the chain rule we then have the following derivative formulas
\begin{align}
    \frac{\partial}{\partial x_{ij}}F_{ij} &= F_{ij} - F_{ij}^2 \\
    \frac{\partial}{\partial x_{ik}}F_{ij} &= -F_{ij}F_{ik} \text{ for any } k \neq j \\
    \frac{\partial}{\partial x_{kl}}F_{ij} &= 0 \text{ for any } k \neq i \text{ and } l \neq j.
\end{align}
We can then express the gradient as
\begin{equation}
\nabla\mathbf{softmax}(A) = 
\begin{bmatrix}
\nabla F_{11}\\
\vdots \\
\nabla F_{1N}\\
\nabla F_{21} \\
\vdots \\
\vdots \\
\nabla F_{NN}
\end{bmatrix}
\end{equation}
where
\begin{equation}
\nabla F_{ij} =
\begin{bmatrix}
-F_{ij}F_{i1}^2 & -F_{ij}F_{i2} & \cdots & F_{ij} - F_{ij}^2 & \cdots & F_{ij}F_{iN}.
\end{bmatrix}
  \end{equation}
From these computations we see that 
\begin{equation}
    ||\nabla \mathbf{softmax}(A)||_F^2 = ||\nabla F_{11}||_F^2 + \cdots + ||\nabla F_{NN}||_F^2.
\end{equation}
We will proceed by bounding each collection
$||\nabla F_{i1}||_F^2 + \cdots + ||\nabla F_{1N}||_F^2$ separately then add up all the bounds. We have
\begin{align}
||\nabla F_{i1}||_F^2 + \cdots + ||\nabla F_{1N}||_F^2 &= 
|F_{i1} - F_{i1}^2|^2 + |F_{i1}F_{i2}|^2 + \cdots + |F_{i1}F_{iN}|^2 \\
&\hspace{0.5cm} +  |F_{i2}F_{i1}|^2 + |F_{i2}-F_{i2}^2|^2 + \cdots + |F_{i2}F_{iN}|^2 \\
&\hspace{0.5cm} + \cdots\cdots\cdots\cdots + \\
&\hspace{0.5cm} +
|F_{iN}F_{i1}|^2 + |F_{iN}F_{i2}|^2 + \cdots + |F_{iN}-F_{iN}^2|^2 \\
&\leq (F_{i1})^2(|1 - F_{i1}| + |F_{i2}| + \cdots + |F_{iN}|)^2 \\
&\hspace{0.5cm}(F_{i2})^2(|F_{i1}| + |1-F_{i2}| + \cdots + |F_{iN}|)^2 \\
&\hspace{0.5cm} + \cdots\cdots\cdots\cdots + \\
&\hspace{0.5cm} + (F_{iN})^2(|F_{i1}| + |F_{i2}| + \cdots + |1-F_{iN}|)^2.
\end{align}
We then observe that since $F_{i1} + \cdots + F_{iN} = 1$ we have that
$1 - F_{ij} = 2(F_{i1}+\cdots + \widehat{F_{ij}}+ \cdots + F_{iN})$ where $\widehat{F_{ij}}$ means we don't include $F_{ij}$ in the sum. This means we get the bound
\begin{align}
 ||\nabla F_{i1}||_F^2 + \cdots + ||\nabla F_{1N}||_F^2 &\leq 
 4F_{i1}^2(\widehat{F_{i1}} + F_{i2}+\cdots + F_{iN}) \\
 &\hspace{0.5cm} + \cdots\cdots\cdots\cdots + \\
 &\hspace{0.5cm} +4F_{iN}^2(F_{i1}+ F_{i2}+\cdots + \widehat{F_{iN}} ) \\
 &\leq 4(F_{i1}^2 + \cdots F_{iN}^2) \\
 &= 4.
\end{align}
Putting all the bounds together for each of the terms $N$ terms
$||\nabla F_{i1}||_F^2 + \cdots + ||\nabla F_{1N}||_F^2$ we get
\begin{equation}
    ||\nabla \mathbf{softmax}(A)||_F^2 \leq 4N
\end{equation}
and this implies
\begin{equation}
    ||\nabla \mathbf{softmax}(A)||_F \leq 2\sqrt{N}.
\end{equation}
This finishes the proof of theorem \ref{thm:softmax_regularity}.

\end{proof}

\subsubsection{Proofs for theorems section  \ref{subsec:poly_acts}}\label{app;poly_proofs}

In this section we will give the proof of theorems \ref{thm:expectation} and \ref{thm:grad_expectation}.

\begin{proof}[Proof of theorem \ref{thm:expectation}]

We will split the matrix product $XQK^TX^T$ and think of it as the product of two matrices. Suppose $\mathbf{A} \in \R^{N\times D} \sim \mathcal{N}(0,\sigma_1^2)$, $\mathbf{B} \in \R^{D\times N}\sim \mathcal{N}(0,\sigma_2^2)$ and $\mathbf{C} = \mathbf{AB}$. Each element in the matrix $\mathbf{C}$ can be written as a product of a row of 
$\mathbf{A}$ with a column of $\mathbf{B}$. Since expectation is linear, we need to compute the expectation of each of these elements. We do the case of the entry $c_{11}$ which is the entry in $\mathbf{C}$ in the first row and first column. For the $p = 1$ case we can then compute

\begin{equation}
    \begin{aligned}
        \mathop{\mathbb{E}}( c_{11}^2 ) & = \mathop{\mathbb{E}}( (\sum_{i=1}^{D} a_{1i}b_{i1})^2 ) \\
        & = \mathop{\mathbb{E}}( \sum_{i=1}^{D} a_{1i}^2b_{i1}^2 + \sum_{i=1}^{D} \sum_{j=1,j\neq i}^{D}a_{1i}b_{i1}a_{1j}b_{j1} )\\
        & = \sum_{i=1}^{D} \mathop{\mathbb{E}}(a_{1i}^2)\mathop{\mathbb{E}}(b_{i1}^2) + \sum_{i=1}^{D} \sum_{j=1,j\neq i}^{D}\mathop{\mathbb{E}}(a_{1i})\mathop{\mathbb{E}}(b_{i1})\mathop{\mathbb{E}}(a_{1j})\mathop{\mathbb{E}}(b_{j1})\\
        & = D\sigma_1^2\sigma_2^2 + 0.
    \end{aligned}
\end{equation}

The Frobenius norm of the matrix $\mathbf{C}$ is just the sum of these values for all $N^2$ elements and this proves the $p=1$ case.

For the case that $p > 1$ we proceed in a similar way. The key observation is that odd powers, in the matrix expansion, will have expectaion $0$, so we need only consider the even powers. Therefore, suppose $\mathbf{C} = (\mathbf{A}\mathbf{B})^p$. We will compute the expectation of the first entry $c_11 \in \mathbf{C}$:

\begin{equation}
    \begin{aligned}
        \mathop{\mathbb{E}}( c_{11}^2 ) & = \mathop{\mathbb{E}}( (\sum_{i=1}^{D} a_{1i}b_{i1})^{2p} ) \\
        & = \mathop{\mathbb{E}}( \sum_{i=1}^{D} a_{1i}^{2p}b_{i1}^{2p} + \sum_{i=1}^{D} \sum_{j=1,j\neq i}^{D}a_{1i}^{2p-2}b_{i1}^{2p-2}a_{1j}^{2}b_{j1}^{2} + \cdots ).
    \end{aligned}
\end{equation}

Note that the first term only has a count of $D$ and the second term has a count of $D(D-1)$. Thus, we only need to consider the $\mathcal{O}(D^{p})$ term where all the components have a power of 2. The count is similar to choosing $p$ items from $D$,

\begin{equation}
    \begin{aligned}
        \mathop{\mathbb{E}}( c_{11}^2 ) & \approx \mathop{\mathbb{E}}(\sum_{\{i_1,\dots,i_p\} \in \{1,\dots,D\}} \prod_{k=1}^p a_{1,i_k}^{2}b_{i_k,1}^{2}) \\
         & = \left(\begin{array}{c} D \\ p \end{array} \right)\frac{2p!}{2^p} \sigma_1^{2p}\sigma_2^{2p} \\
         & = \frac{D!}{(D-p)!}\frac{2p!}{p!2^p} \sigma_1^{2p}\sigma_2^{2p} \\
         & = \frac{D!}{(D-p)!}\frac{2p!}{2p!!} \sigma_1^{2p}\sigma_2^{2p} \\
         & = \frac{D!}{(D-p)!}(2p-1)!! \sigma_1^{2p}\sigma_2^{2p}.
    \end{aligned}
\end{equation}

$\frac{D!}{(D-p)!}$ can always be bounded above by $D^{p}$, so the expectation can be upper bounded by $D^{p}(2p-1)!! \sigma_1^{2p}\sigma_2^{2p}$ and thus we get a quantity of the form
$\mathcal{O}(N)$.

\end{proof}


\begin{proof}[Proof of theorem \ref{thm:grad_expectation}]
    We will do the $p =1$ case first. We proceed similar to the proof of Theorem \ref{thm:expectation}.
\begin{align}
\mathop{\mathbb{E}}(\|\frac{\partial XQK^TX^T}{\partial Q}\|^2_F) 
&= \sum_{i=1}^{N} \sum_{j=1}^{N} \mathop{\mathbb{E}}(   |\frac{\partial x_{i}^T Q K^Tx_{j}}{\partial Q}|^2_F ) \\
&= \sum_{i=1}^{N} \sum_{j=1}^{N} \mathop{\mathbb{E}}(   \|x_{i}x_{j}^TK \|^2_F ) \\
&= \sum_{i=1}^{N} \sum_{j=1}^{N} \mathop{\mathbb{E}}(   \sum_{k=1}^{D} \sum_{l=1}^{d} (x_{ik}\sum_{m=1}^{D} x_{jm}k_{ml})^2 ) \\
&= \sum_{i=1}^{N} \sum_{j=1}^{N} \mathop{\mathbb{E}}(   \sum_{k=1}^{D} \sum_{l=1}^{d} x_{ik}^2 (\sum_{m=1}^{D} x_{jm}k_{ml})^2 ) \\
&= \sum_{i=1}^{N} \sum_{j=1}^{N} \mathop{\mathbb{E}}(   \sum_{k=1}^{D} \sum_{l=1}^{d} x_{ik}^2 (\sum_{m=1}^{D} x_{jm}^2k_{ml}^2+\sum_{m=1}^{D}\sum_{n=1,n\neq m}^{D}x_{jm}k_{ml}x_{jn}k_{nl}) ) \\
        &= \sum_{i=1}^{N} \sum_{j=1}^{N} \sum_{k=1}^{D} \sum_{l=1}^{d} (\sum_{m=1}^{D} \mathop{\mathbb{E}} (x_{ik}^2x_{jm}^2k_{ml}^2) \\ 
        &\hspace{1cm} 
        +\sum_{m=1}^{D}\sum_{n=1,n\neq m}^{D} \mathop{\mathbb{E}} (x_{ik}^2x_{jm}k_{ml}x_{jn}k_{nl}) ) \\
        &= \sum_{i=1}^{N} \sum_{j=1}^{N} \sum_{k=1}^{D} \sum_{l=1}^{d} (\sum_{m=1}^{D} \mathop{\mathbb{E}} (x_{ik}^2x_{jm}^2k_{ml}^2) + 0) \\
        &= \sum_{i=1}^{N} \sum_{j=1}^{N} \sum_{k=1}^{D} \sum_{l=1}^{d}\sum_{m=1}^{D} \mathop{\mathbb{E}} (x_{ik}^2x_{jm}^2k_{ml}^2) \\
        &= \sum_{i=1}^{N} \sum_{k=1}^{D} \sum_{l=1}^{d} \mathop{\mathbb{E}} (x_{ik}^2x_{ik}^2k_{kl}^2) \\
        &\hspace{1cm}
        + \sum_{i=1}^{N} \sum_{j=1,j\neq i}^{N} \sum_{k=1}^{D} \sum_{l=1}^{d}\sum_{m=1,m\neq k}^{D} \mathop{\mathbb{E}} (x_{ik}^2x_{jm}^2k_{ml}^2) \\
        &= NDd3\sigma_x^4\sigma_w^2 + N(N-1)D(D-1)d\sigma_x^4\sigma_w^2 \\
        &\approx N^2D^2d\sigma_x^4\sigma_w^2.
\end{align}
When $p > 1$ we can proceed in a similar way.
\begin{align}
\mathop{\mathbb{E}}(   \|\frac{\partial(XQK^TX^T)^p}{\partial Q}\|^2_F ) &= \sum_{i=1}^{N} \sum_{j=1}^{N} \mathop{\mathbb{E}}(   \|\frac{(\partial x_{i}^TQK^Tx_{j})^p}{\partial Q}\|^2_F ) \\
        &= \sum_{i=1}^{N} \sum_{j=1}^{N} \mathop{\mathbb{E}}(   \|p(x_{i}^TQK^Tx_{j})^{p-1}\frac{\partial x_{i}^TQK^Tx_{j}}{\partial Q}\|^2_F ) \\
        &= \sum_{i=1}^{N} \sum_{j=1}^{N} \mathop{\mathbb{E}}(   \|p(x_{i}^TQK^Tx_{j})^{p-1} x_{i}x_{j}^TK \|^2_F ) \\
        &= \sum_{i=1}^{N} \sum_{j=1}^{N} \mathop{\mathbb{E}}(  p^2(x_{i}^TQK^Tx_{j})^{2p-2} \sum_{k=1}^{D} \sum_{l=1}^{d} (x_{ik}\sum_{m=1}^{D} x_{jm}k_{ml})^2 ).
\end{align}
We know that
\begin{align}
(x_{i}^TQK^Tx_{j})^{2p-2} &= (\sum_{l=1}^{d}( (\sum_{k=1}^{D} x_{ik}q_{kl})\cdot (\sum_{m=1}^{D} x_{jm}k_{ml}) ) )^{2p-2} \\
&= (\sum_{l=1}^{d}\sum_{k=1}^{D}\sum_{m=1}^{D} x_{ik}q_{kl}x_{jm}k_{ml})^{2p-2} \\
        &= (\sum_{k=1}^{D}\sum_{m=1}^{D} x_{ik}x_{jm}\sum_{l=1}^{d}q_{kl}k_{ml})^{2p-2} \\
        &= (\sum_{k=1}^{D}\sum_{m=1}^{D} x_{ik}x_{jm}a_{km})^{2p-2},  
        \end{align}
where $a_{km} = \sum_{l=1}^{d}q_{kl}k_{ml}$. Let $z_{ij}=\sum_{k=1}^{D}\sum_{m=1}^{D} x_{ik}x_{jm}a_{km}$ Thus we have
\begin{align}
\mathop{\mathbb{E}}(   \|\frac{\partial(XQK^TX^T)^p}{\partial Q}\|^2_F ) &= p^2 \sum_{i=1}^{N} \sum_{j=1}^{N} \mathop{\mathbb{E}}(  z_{ij}^{2p-2} \sum_{k=1}^{D} \sum_{l=1}^{d} (x_{ik}\sum_{m=1}^{D} x_{jm}k_{ml})^2 ) \\
        &= \sum_{i=1}^{N} \sum_{j=1}^{N} \sum_{k=1}^{D} \sum_{l=1}^{d} (\sum_{m=1}^{D} \mathop{\mathbb{E}} (z_{ij}^{2p-2}x_{ik}^2x_{jm}^2k_{ml}^2) \\ 
        & +\sum_{m=1}^{D}\sum_{n=1,n\neq m}^{D} \mathop{\mathbb{E}} (z_{ij}^{2p-2}x_{ik}^2x_{jm}k_{ml}x_{jn}k_{nl}) ) \\
        &= \sum_{i=1}^{N} \sum_{j=1}^{N} \sum_{k=1}^{D} \sum_{l=1}^{d} (\sum_{m=1}^{D} \mathop{\mathbb{E}} (z_{ij}^{2p-2}x_{ik}^2x_{jm}^2k_{ml}^2) \\
        & +\sum_{m=1}^{D}\sum_{n=1,n\neq m}^{D} \mathop{\mathbb{E}} (z_{ij}^{2p-3}x_{ik}^2x_{jm}^2k_{ml}^2x_{jn}^2k_{nl}^2) ) \\
        &\approx N^2Dd (D  (D^{2p-2}d^{p-1}(2p-3)!!\sigma_x^{4p}\sigma_w^{4p-2} ) + 0 \\
        &=  N^2 D^{2p}d^{p}(2p-3)!!\sigma_x^{4p}\sigma_w^{4p-2} 
\end{align}
showing that we can bound the gradient by a quantity of the form
$\mathcal{O}(N)$
and the proof is complete.

\end{proof}

\subsection{Experiments}\label{app:exps}

\subsubsection{Hardware}

The vision transformer experiments in \cref{subsec:image_classi}, the object detection and instance segmentation experiments in \cref{subsec:object_detec} and the Nystr\"omformer experiments from \cref{subsec:nlp_exps} were all carried out on Nvidia A100 GPUs. The physics informed transformer experiments in \cref{sec:pinns} were carried out on a Nvidia A6000 GPU.

\subsubsection{Experimental hyperparameters}\label{app:exps_hyps}

\paragraph{Vision transformers in \cref{subsec:image_classi}.} In \cref{subsec:image_classi} we tested four different vision transformers, ViT-B \cite{dosovitskiy2020image}, DeiT-B \cite{touvron2021training}, Swin-B \cite{liu2021swin} and XCiT-M \cite{ali2021xcit}, with the activations softmax, $x^3$ + fixed scaling,  $x^3$ + dynamic scaling and $x^3$.  The training strategy follow the exact strategy used in each of the original papers, we used the Timm \cite{rw2019timm} libraries to train our models.

\paragraph{Object detection and instance segmentation in \cref{subsec:object_detec}.} Each of the models are trained for 36 epochs using the AdamW optimizer with learning rate of $10^{-4}$, $0.05$ weight decay and batch size of $16$.

\paragraph{Nystr\"{o}mformer in \cref{subsec:nlp_exps}.} For the Nystr\"{o}mformer experiments we used the exact same training strategy form the original paper \cite{xiong2021nystromformer}. Model were trained using the GitHub provided by the authors of \cite{xiong2021nystromformer}.

\subsubsection{Frobenius norm computations}\label{app:frob_norm_complete}

In \cref{subsec:testing_theory} we showed plots of the Frobenius norm of the self-attention matrix and for the Jacobian of the self-attention matrix for softmax, $\frac{1}{14}x^3$, and $x^3$. This was done for a ViT-Tiny architecture on the Tiny-ImageNet dataset. \cref{fig:tiny_frob_12} shows the plots of the Frobenius norm of the self-attention matrix for the ViT-Tiny architecture, during training, for all layers averaged over the heads within each layer. \cref{fig:jac_tiny_frob_12} shows the Frobenius norm of the Jacobian of the self-attention matrix during training for each layer, averaged over the total number of heads within each layer.

\begin{figure}[ht!]
    \centering
    \includegraphics[width=1.\linewidth]
    {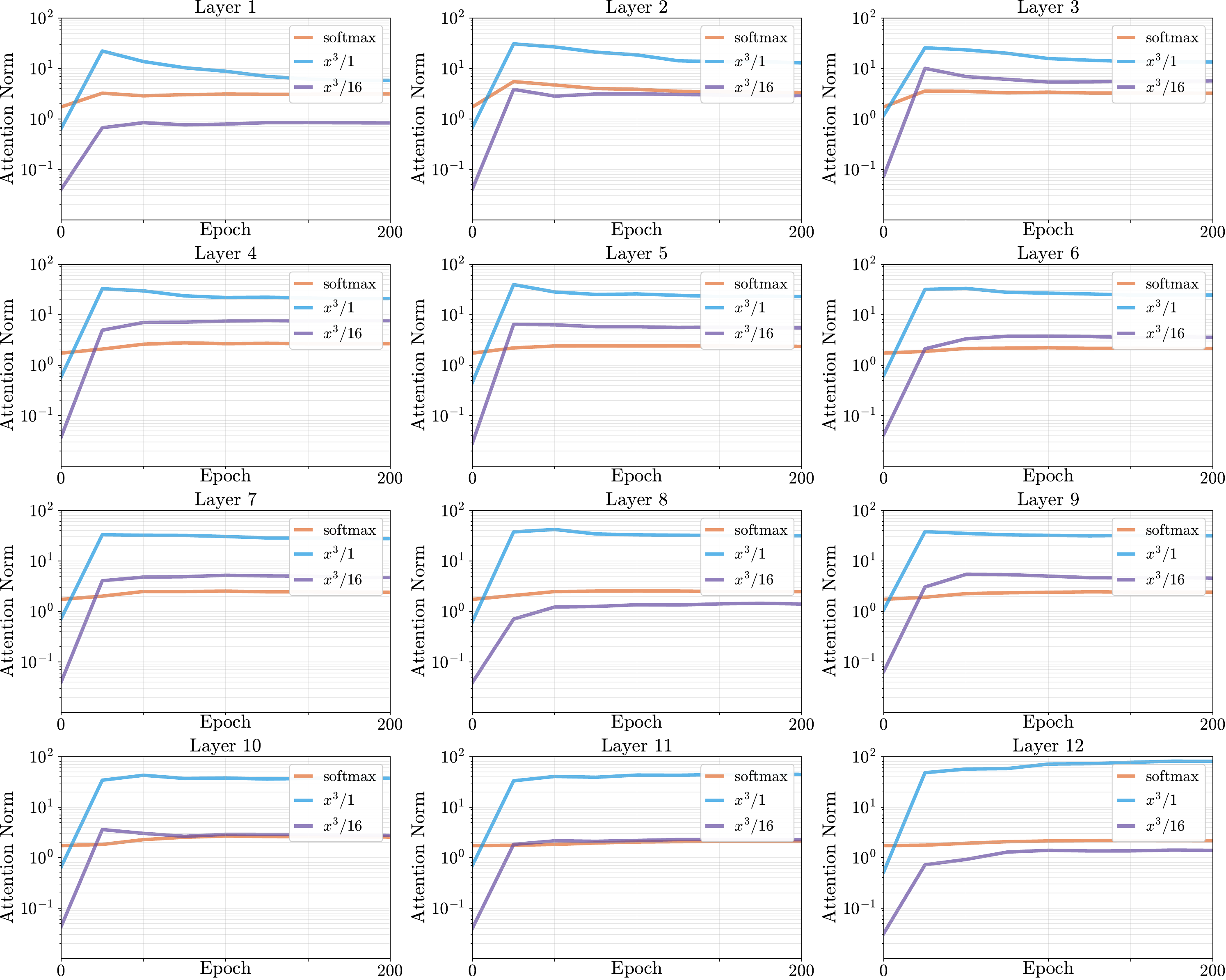}
    \caption{Frobenius norm of self-attention matrix for softmax, $\frac{1}{16}x^3$ and $x^3$ on ViT-Tiny during training on the Tiny-ImageNet dataset (zoom in for better viewing).}
    \label{fig:tiny_frob_12}
\end{figure}

\begin{figure}[ht!]
    \centering
    \includegraphics[width=1.\linewidth]
    {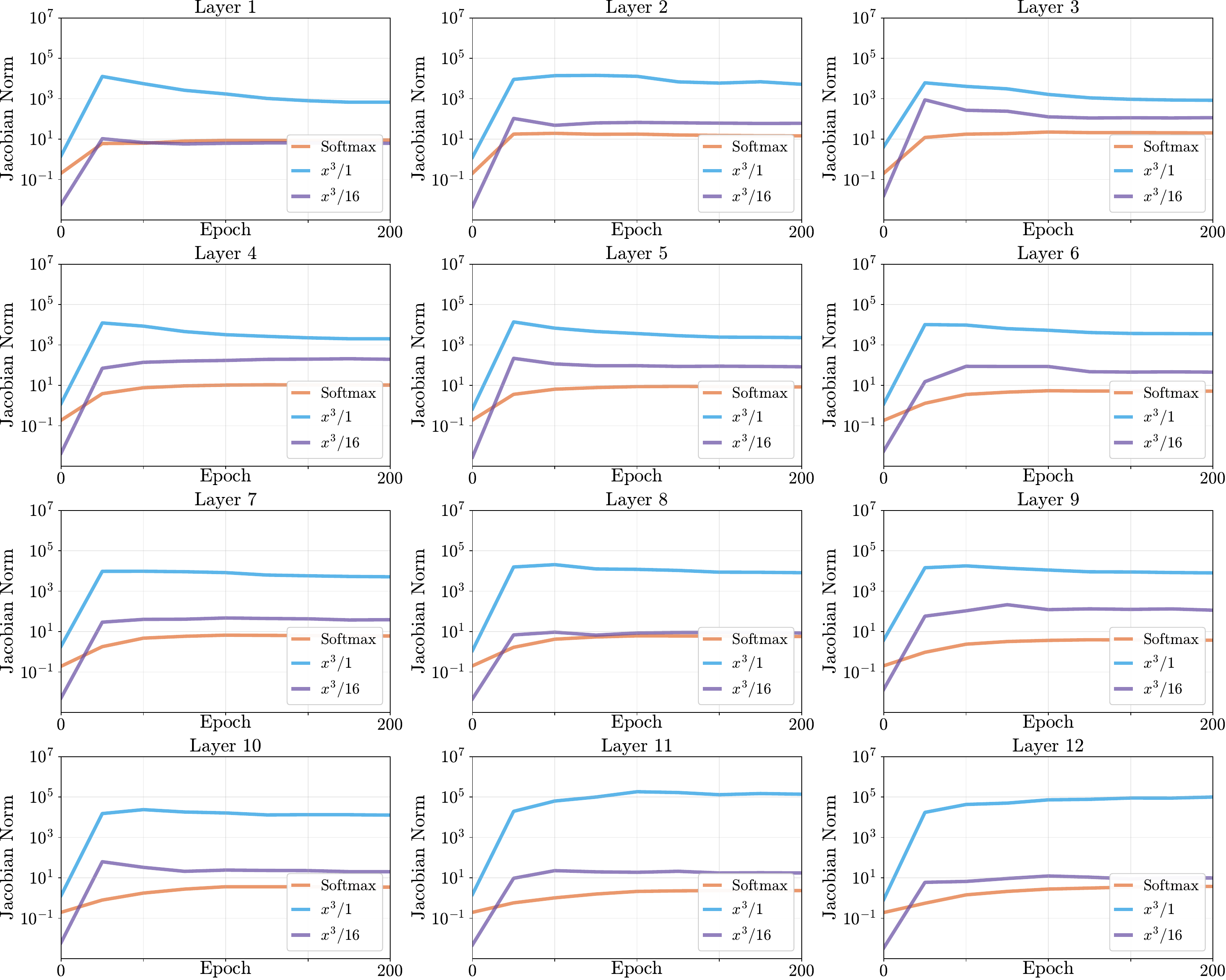}
    \caption{Frobenius norm of the Jacobian of the self-attention matrix for softmax, $\frac{1}{16}x^3$ and $x^3$ on ViT-Tiny during training on the Tiny-ImageNet dataset (zoom in for better viewing).}
    \label{fig:jac_tiny_frob_12}
\end{figure}

\subsubsection{Discussion on Linear attention and other polynomials on ViTs}\label{app:other_polys}

In the experiments \cref{sec:exps} we showed our insights using the polynomial $x^3$ as this was a non-linear polynomial that did not satisfy the key properties that softmax did. In this section we will compare softmax with other polynomials that we also found to perform well.

The results for ViTs is shown in \cref{tab:vits_all_poly}. As we can see from that table each polynomial does far better when scaled using the theory from \cref{sec:theory}. We note that when we tried to train polynomial higher than order $6$ they did not train well. On inspecting the weights we found that several were very close to zero leading to a gradient vanishing problem. We hypothesize that this is because the transformers were randomly initialized with weights about zero. Thus taking such a weight and applying a polynomial of the form $x^k$ with $k \geq 6$ would make those weights orders of magnitude smaller, making it difficult to train after some point. 

\begin{table}[!ht]
    \caption{Comparisons of pre-training models with different activation functions on ImageNet-1k. We report top-1 classification accuracy (\%).}
    \centering
    \scalebox{1.0}{
    \begin{tabular}{ccccc}
        \multirow{2}{*}{} & \multicolumn{4}{c}{Models} \\
        & ViT-B & DeiT-B & Swin-B & XciT-M \\
        \midrule
        softmax & 80.3 & 81.5 & 83.5 & 82.7 \\
        \midrule
        $x$ + fixed scale & 78.4 & 79.4 & 80.6 & 80.2 \\
        \midrule
        $x$ + dynamic scale & 78.7 & 79.5 & 80.7 & 80.4 \\
        \midrule
        $x$ & 74.1 & 77.9 & 78.3 & 78.1 \\
        \midrule
        $x^2$ + fixed scale & 80.1 & 81.5 & 83.4 & 82.5 \\
        \midrule
        $x^2$ + dynamic scale & 80.3 & 81.6 & 83.5 & 82.7 \\
        \midrule
        $x^2$ & 77.8 & 78.2 & 79.8 & 79.4 \\
        \midrule
         $x^4$ + fixed scale & 80.3 & 81.5 & 83.7 & 82.7 \\
        \midrule
        $x^4$ + dynamic scale & 80.3 & 81.6 & 83.7 & 82.8 \\
        \midrule
        $x^4$ & 77.9 & 78.6 & 79.9 & 79.6 \\
        \midrule
           $x^5$ + fixed scale & 80.3 & 81.4 & 83.4 & 82.5 \\
        \midrule
        $x^5$ + dynamic scale & 80.3 & 81.5 & 83.4 & 82.6 \\
        \midrule
        $x^5$ & 77.7 & 78.0 & 79.4 & 79.5 \\
        \midrule
    \end{tabular}
    }
    \label{tab:vits_all_poly}
\end{table}



\subsubsection{Physics informed transformers}\label{sec:pinns}

In \cite{zhao2023pinnsformer} transformers for modeling partial differential equations (PDEs) were introduced and shown to yield superior performance on a variety of PDEs. In this section we compare softmax with the polynomial activation we have been using in the previous section on two common PDEs within the literature, namely the convection PDE and the 1d-Reaction PDE. 

\paragraph{Convection PDE:} This is a 1-dimensional PDE that models transport phenomena. The PDE system is defined by the following equations
\begin{align}
    \frac{\partial u}{\partial t} 
    + \beta\frac{\partial u}{\partial x} = 0, \text{ for } 
    x \in [0, 2\pi], t \in [0,1] \\
IC: u(x,0) = sin(x), BC: u(0, t) = u(2\pi, t)    
\end{align}
where $\beta$ is the convection coefficient and is set as 
$\beta = 50$. IC stands for initial condition and BC for boundary condition. For more details on this system consult \cite{zhao2023pinnsformer}.

\paragraph{1d-reaction PDE:} The 1d-reaction equation is commonly used in modeling chemical reactions. It is defined by the following system.
\begin{align}
 \frac{\partial u}{\partial t} - \rho u(1-u) = 0  \text{ for } 
    x \in [0, 2\pi], t \in [0,1] \\
IC: u(x,0) = exp\big{(}\frac{(x-\pi)^2}{2(\pi/4)^2} \big{)}, 
BC: u(0,t) = u(2\pi, t)
\end{align}
where $\rho$ is known as the reaction coefficient and is set as $\rho=5$. For more details on this system consult \cite{zhao2023pinnsformer}.

As in \cite{zhao2023pinnsformer} we considered a physics informed transformer (PINNsformer) on the above two PDEs. We compared softmax with $x^3$ + fixed scaling, $x^3$ + dynamic scaling, where the scale was initialized to $\frac{1}{\sqrt{51}}$ as the input sequence length is $51$, and $x^3$. As in \cite{zhao2023pinnsformer} we computed the training loss, and for testing the relative mean absolute error (rMAE), and the relative mean squared error (rMSE). \cref{tab:convection_1d_reaction} shows the results from which it can be seen that both $x^3$ + fixed scaling and $x^3$ + dynamic scaling are competitive with the original softmax.

\begin{table*}[!ht]
\caption{Comparison of different activations on PINNsformer for  Convection and 1D-Reaction problems.}
\centering
\scalebox{0.9}{
\begin{tabular}{lcccccc}
\hline
\rowcolor{mygrey}
\multirow{2}{*}{} & \multicolumn{3}{c|}{Convection PDE} & \multicolumn{3}{c}{1D-Reaction PDE} \\ 
\cline{2-7}
\rowcolor{mygrey}
& Loss & rMAE & rMSE & Loss & rMAE & rMSE \\ 
\hline
softmax & 5.4e-5 & 4.1e-2 & 4.4e-2 & 3.8e-6 & 3.1e-2 & 5.8e-2 \\ 
$x^3$ + fixed scale & \cellcolor{myorange}4.9e-5 & \cellcolor{myorange}3.9e-2 & \cellcolor{myorange}4.3e-2 & \cellcolor{myorange}3.1e-6 
& \cellcolor{myorange}3.0e-2 & \cellcolor{myorange}5.6e-2 \\ 
$x^3$ + dynamic scale & \cellcolor{mygreen}4.1e-5 & \cellcolor{mygreen}3.5e-2 & \cellcolor{mygreen}4.0e-2 & \cellcolor{mygreen}2.8e-6 
& \cellcolor{mygreen}2.7e-2 & \cellcolor{mygreen}5.4e-2 \\ 
$x^3$ & 7.2e-5 & 9.8e-2 & 1.1e-1 & 6.2e-6 & 6.1e-2 & 9.8e-2 \\ 
\hline
\end{tabular}
}
\label{tab:convection_1d_reaction}
\end{table*}

The training of the physics informed transformers followed the exact recipe in the original paper 
\cite{zhao2023pinnsformer} with their data and training schemes available from their GitHub \cite{pinnsformer_github}. The evaluation metrics are defined by

\begin{equation}
\text{rMAE} = \frac{\sum_{n=1}^N \left| \hat{u}(x_n, t_n) - u(x_n, t_n) \right|}{\sum_{n=1}^{N} \left| u(x_n, t_n) \right|}
\end{equation}

\begin{equation}
\text{rMSE} = \sqrt{\frac{\sum_{n=1}^N \left| \hat{u}(x_n, t_n) - u(x_n, t_n) \right|^2}{\sum_{n=1}^{N} \left| u(x_n, t_n) \right|^2}}
\end{equation}

where \( N \) is the number of testing points, \( \hat{u} \) is the neural network approximation, and \( u \) is the ground truth.




\end{document}